\documentclass[letterpaper]{article}
\usepackage{aaai}
\usepackage{times}
\usepackage{helvet}
\usepackage{courier}
\usepackage{amsthm}
\usepackage[usenames,dvipsnames]{color}
\usepackage{graphicx}
\usepackage{xspace}

\newcommand{\papertitle}{Stable Model Counting and Its Application in Probabilistic Logic Programming}

\newtheorem{theorem}{Theorem}
\newtheorem{example}{Example}
\newtheorem{definition}{Definition}
\newtheorem{corollary}[theorem]{Corollary}
\newtheorem{proposition}[theorem]{Proposition}

\newcommand{\printsolver}[1]{\textsc{#1}\xspace}

\newcommand{\dlv}{\printsolver{dlv}}
\newcommand{\clasp}{\printsolver{clasp}}

\newcommand{\Problog}{\printsolver{Problog2}}
\newcommand{\ASProblog}{\printsolver{ASProblog}}
\newcommand{\ASProblogS}{\printsolver{ASProblogS}}
\newcommand{\dsharp}{\printsolver{Dsharp}}
\newcommand{\sharpsat}{\printsolver{sharpSAT}}

\newcommand{\true}{\mathit{true}}
\newcommand{\false}{\mathit{false}}

\newcommand{\vars}{\mathit{vars}}
\newcommand{\countt}{\mathit{count}}
\newcommand{\shadow}{\mathit{copy}}
\newcommand{\project}{\mathit{prj}}
\newcommand{\projectr}{\mathit{prj_{copy}}}

\newcommand{\Founded}{\VV_{\cal F}}
\newcommand{\Standard}{\VV_{\cal S}}
\newcommand{\VV}{{\cal V}}

\newcommand{\JA}{\mathit{JA}}

\newcommand{\Least}{\mathit{Least}}
\newcommand{\stress}{\mathit{stress}}
\newcommand{\influences}{\mathit{influences}}
\newcommand{\inn}{\mathit{in}}
\newcommand{\reach}{\mathit{reach}}
\newcommand{\node}{\mathit{node}}
\newcommand{\edge}{\mathit{edge}}

\newcommand{\residual}[2]{#1|_{#2}}
\newcommand{\justres}[2]{#1|^{\mathit{j}}_{#2}}

\newcommand{\aspsat}{ASP-SAT{~}}
\newcommand{\GraphReliability}{\textsl{GraphRel}\xspace}
\newcommand{\SmokersFriends}{\textsl{SmokersFriends}\xspace}

\newcommand{\ddnnf}{d-DNNF\xspace}

\newcommand{\ignore}[1]{}

\newcommand{\noproofs}[1]{}
\renewcommand{\noproofs}[1]{#1}

\newcommand{\authornames}{Rehan Abdul Aziz \and Geoffrey Chu \and Christian Muise \and Peter Stuckey \\
National ICT Australia, Victoria Laboratory \thanks{NICTA is funded by the Australian Government as represented by the
Department of Broadband, Communications and the Digital Economy and the
Australian Research Council through the ICT Centre of Excellence program.} \\ 
Department of Computing and Information Systems \\ The University of Melbourne\\ }

\begin{document}

%
\title{\papertitle}
\author{\authornames}
\maketitle
\begin{abstract}
\begin{quote}
Model counting is the problem of computing the number of models that satisfy a given propositional theory. 
It 
has recently been applied to solving inference tasks in probabilistic logic programming, where the goal is to compute the probability of given queries being true provided a set of mutually independent random variables, a model (a logic program) and some evidence. The core of solving this inference task involves translating the logic program to a propositional theory and using a model counter. 
In this paper, we show that for some problems that involve inductive definitions like reachability in a graph, the translation of logic programs to SAT can be expensive for the purpose of solving inference tasks. For such problems, direct implementation of stable model semantics allows for more efficient solving. We present two implementation techniques, based on unfounded set detection, that extend a propositional model counter to a stable model counter. Our experiments show that for particular problems, our approach can outperform a state-of-the-art probabilistic logic programming solver by several orders of magnitude in terms of running time and space requirements, and can solve instances of significantly larger sizes on which the current solver runs out of time or memory. 
\end{quote}
\end{abstract}

\section{Introduction}

Consider the counting version of graph reachability problem: given a directed graph, count the number
of subgraphs in which node $t$ is reachable from node $s$ \cite{valiant_reliability}.
This problem can be naturally modeled as a logic program under stable model semantics \cite{stable_models}. 
Let us say that the input is given by two predicates:
$\node(X)$ and $\edge(X,Y)$. For each node, we can introduce a decision variable $\inn$
that models whether the node is in the subgraph. Furthermore, we can model reachability ($\reach$) from $s$ as an inductive definition using
the following two rules: $\reach(s) \leftarrow \inn(s)$ and $\reach(Y) \leftarrow \inn(Y), \reach(X), \edge(X,Y)$. The first
rule says that $s$ itself is reachable if it is in the subgraph. The second rule is the inductive case, specifying
that a node $Y$ is reachable if it is in the subgraph and there is a reachable node $X$ that has an edge to it.
Additionally, say there are arbitrary constraints in our problem, e.g., only consider subgraphs where a certain $y$ is also reachable from $s$ etc.
This can be done using the integrity constraint: $\leftarrow \neg reach(y)$. The number of stable models of this program
is equal to the number of solutions of the problem.

There are at least two approaches to counting the stable models of a logic program.
The first is to translate a given logic program to a propositional theory such that there is a one-to-one
correspondence between the propositional models of the translated program and the stable models of the original program, and use SAT model counting
\cite{model_counting}. We show that this approach does not scale well in practice
since such translations, if done \emph{a priori}, can grow exponentially with the input size.
The second approach is to use an \emph{answer set programming} (ASP) solver like \clasp \cite{clasp} or \dlv \cite{dlv}
and enumerate all models. This approach is extremely inefficient since model counting algorithms
have several optimizations like caching and dynamic decomposition that are not present in ASP solvers.
This motivates us to build a stable model counter that can take advantage of state-of-the-art ASP technology
which combines partial translation and lazy unfounded set \cite{vangelder} detection.
However, we first show that it is not correct to naively combine partial translation and lazy unfounded set detection with SAT
model counters due to the aforementioned optimizations in model counters.
We then suggest two approaches to properly integrate unfounded set propagation in a model counter.

We show that we can apply our algorithms to solve probabilistic logic programs \cite{problog_concepts}.
Consider the probabilistic version of the above problem, also called the
\emph{graph reliability} problem \cite{graph_reliability}. In this version, each node can be in the subgraph with a
certain probability $1 - p$, or equivalently, fail with the probability $p$. We can model this by simply
attaching probabilities to the $\inn$ variables. We can model observed \emph{evidence} as constraints. E.g.,
if we have evidence that a certain node $y$ is reachable from $s$, then we can model this as the unary \emph{constraint} (not rule): $\reach(y)$.
The goal of the problem is to calculate the probability of node $t$ being reachable from node $s$ given the evidence.
The probabilistic logic programming solver \Problog \cite{problog} approaches this inference task by reducing it to weighted model
counting of the translated propositional theory of the original logic program. We extend \Problog to use our implementation of  
stable model counting on the original logic program 
and show that 
our approach is more scalable.

\section{Preliminaries}

We consider propositional variables $\VV$.
Each $v \in \VV$ is (also) a  positive literal, and $\neg v, v \in \VV$ is a
negative literal.
Negation of a literal, $\neg l$ is $\neg v$ if $l = v$, and $v$ 
if $l = \neg v$. 
An assignment $\theta$ is a set of literals  which represents the literals which are true in the assignment,
where $\forall v \in \VV. \{v, \neg v\} \not\subseteq \theta$.
If $o$ is a formula or assignment, let $vars(o)$ 
be the subset of $\VV$ appearing in $o$. 
Given an assignment $\theta$, let $\theta^+ = \{  v \in \theta ~|~ v \in
\VV\}$ and $\theta^- = \{ \neg v \in \theta ~|~ v \in \VV \}$. 
Two assignments $\theta_1$ and $\theta_2$ \emph{agree on variables $V \subseteq
\VV$}, written $\theta_1 =_{V} \theta_2$, if $vars(\theta_1^+) \cap V = vars(\theta_2^+) \cap V$ and
$vars(\theta_1^-) \cap V = vars(\theta_2^-) \cap V$.
Given a partial assignment $\theta$ and a Boolean formula
$F$, let $\residual{F}{\theta}$ be the \emph{residual} of $F$ w.r.t. $\theta$.
$\residual{F}{\theta}$ is constructed from $F$ by substituting each
literal $l \in \theta$ with $\true$ and each literal $\neg l \in \theta$
with $\false$ and simplifying the resulting formula.
For a formula $F$, $\countt(F)$ is the number of assignments to $\vars(F)$
that satisfy $F$.

\subsection{DPLL-Based Model Counting}

State of the art SAT model counters are very similar to SAT solvers, but
have three important optimisations.  The first optimisation is to count
solution \emph{cubes} (i.e., partial assignments $\theta$, 
$vars(\theta) \subseteq \VV$ whose every extension is a
solution) instead of individual solutions. Consider the Boolean formula:
$F_1 = \{\neg b \vee a$, $\neg c \vee \neg a \vee b$, $\neg d \vee c$, $\neg
e \vee c\}$. Suppose the current partial assignment is $\{a, b, c\}$. The formula is 
already satisfied irrespective of values of $d$ and $e$.
Instead of searching further and finding all 4 solutions, we can stop
and record that we have found a solution cube containing $2^k$ solutions,
where $k$ is the number of unfixed variables.

The second important optimisation is caching. Different partial assignments
can lead to identical subproblems which contain the same number of
solutions. By caching such counts, we can potentially save significant
repeated work. For a formula $F$ and an assignment $\theta$, the number of
solutions of $F$ under the subtree with $\theta$ is given by $2^{|vars(F)| -
  |vars(\theta)| - |vars(F|_{\theta})|} \times
\countt(\residual{F}{\theta})$. We can use the residual as the key and cache
the number of solutions the subproblem has. For example, consider $F_1$
again. Suppose we first encountered the partial assignment $\theta_1 = \{d,
c\}$. Then $\residual{F_1}{\theta_1} = \{\neg b \vee a, \neg a \vee
b\}$. After searching this subtree, we find that this subproblem has 2
solutions and cache this result. The subtree under $\theta_1$ thus has
$2^{5-2-2} \times 2 = 4$ solutions. Suppose we later encounter $\theta_2 = \{\neg
d, e, c\}$. We find that $\residual{F_1}{\theta_2}$ is the same as
$\residual{F_1}{\theta_1}$. By looking it up in the cache, we can see that
this subproblem has 2 solutions. Thus
the subtree under $\theta_2$ has $2^{5-3-2} \times 2 = 2$ solutions.

The last optimisation is dynamic decomposition. Suppose after fixing
some variables, the residual decomposes into two or more formulas involving
disjoint sets of variables. We can count the number of solutions for
each of them individually and multiply them together to get the right
result.  Consider $F_2 = \{a \vee \neg b \vee c, c \vee \neg d \vee e, e
\vee f\}$ and a partial assignment $\{\neg c\}$.  The residual program can be
decomposed into two components $\{a \vee \neg b\}$ and $\{\neg d \vee e, e \vee f\}$ with variables $\{a,b\}$
and $\{\neg d \vee e, e \vee f\}$ respectively. Their counts are 3 and 5 respectively, therefore, the number of
solutions for $F_2$ that extend the assignment $\{\neg c\}$ is
$3 \times 5 = 15$.  The combination of the three optimisations
described above into a DPLL style backtracking algorithm has been shown to
be very efficient for model counting. See
\cite{dpll_with_caching,model_counting,cachet} for more details.

\subsection{Answer Set Programming}

We consider $\VV$ split into two disjoint sets of variables
\emph{founded} variables ($\Founded$) and \emph{standard} variables
($\Standard$).  
An \aspsat \emph{program} $P$ is a tuple $(\VV, R, C)$ where
$R$ is a set of \emph{rules} of form:
$a \leftarrow b_1 \wedge \ldots \wedge b_n \wedge \neg c_1 \wedge \ldots \wedge \neg c_m$ 
such that $a \in \Founded$ and $\{b_1, \ldots, c_m\} \subseteq \VV$ and $C$ is a set of \emph{constraints} over the
variables represented as disjunctive clauses.
A rule is \emph{positive} if its body only contains positive founded literals.
The \emph{least assignment} of a set of positive rules $R$, 
written $\Least(R)$ is one that that satisfies all the rules and contains the least number of positive literals. 
Given an assignment $\theta$ and a program $P$, 
the reduct of $\theta$ w.r.t. $P$, written, $P^\theta$ is
a set of positive rules that is obtained as follows: 
for every rule $r$, if any $c_i \in \theta$, or $\neg b_j \in \theta$ for any
\emph{standard} positive literal, then $r$ is discarded, otherwise, 
all negative literals and standard variables are removed from $r$
and it is included in the reduct. 
An assignment $\theta$ is a stable model of a program $P$ iff it satisfies all its constraints
and $\theta =_{\Founded} \Least(P^\theta)$.
Given an assignment $\theta$ and a set of rules $R$, the residual rules $\residual{R}{\theta}$ are defined similarly to residual clauses
by treating every rule as its logically equivalent clause. 
A program is \emph{stratified} iff it admits a mapping
$level$ from $\Founded$ to non-negative integers 
where for each rule in the program s.t., referring to the above rule form, $level(a) > level(c_i)$ whenever $c_i \in \Founded$ for $1 \leq i \leq m$ and 
$level(a) \geq level(b_i)$ whenever $b_i \in \Founded$ for $1 \leq i \leq n$.
In ASP terms, standard variables, founded variables and constraints are equivalent to
\emph{choice} variables, regular ASP variables, and integrity constraints resp.
We opt for the above representation because it is closer to SAT-based implementation of modern ASP solvers.

\section{SAT-Based Stable Model Counting}
\label{sec:sm_counting}

The most straight forward approach to counting the stable models of a logic
program is to translate the program into propositional theory and use a
propositional model counter. As long as the translation produces a
one-to-one correspondence between the stable models of the program and the
solutions of the translated program, we get the right stable model
count. Unfortunately, this is not a very scalable approach. Translations
based on adding loop formulas~\cite{assat} or the
\emph{proof-based} translation used in \Problog~\cite{problog} require the addition
of an exponential number of clauses in general (see \cite{exponential_loops} and \cite{problog_sdd} respectively).
Polynomial sized translations based on
\emph{level rankings} \cite{lp2sat} do exist, but do not produce a one to
one correspondence between the stable models and the solutions and thus are
inappropriate for stable model counting.

Current state of the art SAT-based ASP solvers do not rely on a full
translation to SAT. Instead, they rely on lazy unfounded set detection. In
such solvers, only the rules are translated to SAT. There is an extra
component in the solver which detects unfounded sets and lazily adds the
corresponding loop formulas to the program as 
required~\cite{clasp_journal}. Such an approach is much more scalable for
solving ASP problems. However, it cannot be naively combined
with a standard SAT model counter algorithm. This is because the SAT model
counter requires the entire Boolean formula to be available so that it can
check if all clauses are satisfied to calculate the residual program. 
However, in this case, the loop formulas are being lazily generated
and many of them are not yet available to the model counter. Naively
combining the two can give the wrong results, as illustrated in the next
example.

\begin{example}
\label{ex:incorrect}
Consider a program $P_1$ with founded variables $\{a, b\}$, 
standard variables $\{s\}$ and rules: $\{a \leftarrow b,
b \leftarrow a, a \leftarrow s\}$. 
There are only two stable models of the program $\{a,b,s\}$ and
$\{\neg a, \neg b, \neg s\}$. 
If our partial assignment is $\{a,b\}$, then the residual program contains
an empty theory which means that the number of solutions extending this assignment is $2$ (or $2^{|\{s\}|}$).
This is clearly wrong, 
since $\{a,b,\neg s\}$ is not a stable model of the program.

Now consider $P_2$ which is equal to $P_1$ with these additions: founded variable $c$, standard variables $\{t,u\}$ and two rules: $c \leftarrow a \wedge t$
and $b \leftarrow u$.
Consider the partial assignment $\{u,a,b,s\}$, the residual program has only one rule: $c \leftarrow t$. It has two
stable models, $\{c,t\}$ and $\{\neg c, \neg t\}$. Now, with the partial assignment $\{\neg u,a,b\}$, we get the same residual program and the number
of solutions should be: $2 \times 2^{|\{s\}|} = 4$ which is wrong since $s$ cannot be false in order for $a,b$ to be true when $u$ is false, i.e., 
$\{\neg u, c,t,a,b,\neg s\}$ and $\{\neg u, \neg c, \neg t,a,b,\neg s\}$ are not stable models of $P_2$.
\end{example}

In order to create a stable model counter which can take advantage of the
scalability of lazy unfounded set detection, we need to do two things: 1)
identify the conditions for which the ASP program is fully satisfied and
thus we have found a cube of stable models, 2) identify what the residual of
an ASP program is so that we can take advantage of caching and dynamic
decomposition.

\subsection{Searching on Standard Variables for Stratified Programs}
\label{sec:fixed}
The first strategy is simply to restrict the search to standard variables. 
If the program is stratified,
then the founded variables of the program are 
functionally defined by the standard variables
of the program. 
Once the standard variables are fixed, all the founded variables are 
fixed through
propagation (unit propagation on \emph{rules} and the unfounded set propagation).
It is important in this approach that the propagation on the founded variables is only carried out on the rules
of the program, and not the constraints. Constraints involving founded variables should only be \emph{checked}
once the founded variables are fixed.
The reader can verify that in Example 1, if we decide on standard variables first, then none of the problems occur.
E.g., in $P_1$, if $s$ is fixed to either true or false, then we do not get any wrong stable model cubes.
Similarly, in $P_2$, if we replace the second assignment with $\{\neg u, s\}$ which propagates $\{a,b\}$,
we still get the same residual program, but in this case, it is correct to use the cached value.
Note that stratification is a requirement for all probabilistic logic programs
under the distribution semantics~\cite{distribution_semantics}. 
For such programs given an assignment to standard variables, 
the well-founded model of the resulting program is the unique stable model.

\subsection{Modifying the Residual Program}
\label{sec:modify}

In ASP solving, it is often very useful to make decisions on founded variables as it can significantly prune the search space.
For this reason, we present a more novel approach to overcome the problem demonstrated in Example \ref{ex:incorrect}.

The root problem in Example \ref{ex:incorrect} in both cases is the failure to distinguish
between a founded variable being true and being \emph{justified}, i.e.,
can be inferred to be true from the rules and current standard and negative literals. 
In the example, in $P_1$, $a$ and $b$ are made true 
by search (and possibly propagation) but they are not justified
as they do not necessarily have externally supporting rules (they are not true under stable model semantics if we set $\neg s$). 
In ASP solvers, this is not a problem since the existing unfounded set detection algorithms guarantee that in \emph{complete} assignments, 
a variable being true implies that it is justified. 
This is not valid for \emph{partial} assignments, which we need for counting stable model cubes.
Next, we show that if we define the residual rules (not constraints) of a
program in terms of justified subset of an assignment, then we can 
leverage a propositional model counter augmented with unfounded set detection to correctly compute stable
model cubes of a program. In order to formalize and prove this, we need further definitions.

Given a program $P = (\VV, R, C)$ and a partial assignment $\theta$, the \emph{justified assignment}
$\JA(P, \theta)$ is the subset of $\theta$ that includes all standard and founded negative literals plus all
the positive founded literals implied by them using the rules of the program. More formally, let 
$J_0(\theta) = \theta^- \cup \{v \in \theta | v \in \Standard\} $. Then,
$\JA(P, \theta) = J_0(\theta) \cup \{v \in \Founded | v \in \theta, 
v \in \Least(\residual{R}{J_0(\theta)})  \}$.

\begin{definition}
\label{def:justres}
Given a program $P = (\VV, R, C)$ and a partial assignment $\theta$, let $J = \JA(P, \theta)$ and $U = \vars(\theta) \setminus \vars(J)$.
The justified residual program of $P$, w.r.t. $\theta$ is written $\justres{P}{\theta}$ and is equal to $(W,S,D)$
where $S = \residual{R}{J}$, $D = \residual{C}{\theta} \cup \{u | u \in U\}$ and $W = \vars(S) \cup \vars(D)$.
\end{definition}

\begin{example}
\label{ex:residual}
Consider a program $P$ with founded variables $\{a,b,c,d,e,f\}$, standard variables $\{s,t,u,x,y,z\}$ and
the following rules and constraints:
$$
\begin{array}{llll}
a \leftarrow b. & c \leftarrow d. & e \leftarrow \neg f. & \neg s \vee \neg t \\
b \leftarrow a. & d \leftarrow u. & f \leftarrow \neg e. & a \vee b \\
a \leftarrow s. &                 &                      & f \vee x \\
b \leftarrow t. &                 &                      & \\
\end{array}
$$
Let $\theta = \{a,b,d,u,\neg e, c, f\}$. Then, $J_0(\theta) = \{u, \neg e\}$
and $\JA(P, \theta) = J_0(\theta) \cup \{d,f,c\}$. 
The justified residual program w.r.t. $\theta$ has all the rules in the first column
and has the constraints: \{$\neg s \vee \neg t$, $a$, $b$\}. 
\end{example}


\newcommand{\thmain}{
Given an \aspsat program $P = (\VV, R, C)$ and a partial assignment $\theta$, let $\justres{P}{\theta} = (W, S, D)$ be denoted by $Q$.
Let the remaining variables be $\VV_r = \VV \setminus (W \cup \vars(\theta))$ and $\pi$ be a complete assignment over $W$.
Assume any founded variable for which there is no rule in $S$ is false in $\theta$.
\begin{enumerate}
\item If $\pi$ is a stable model of $Q$, then for any assignment $\theta_r$ over the remaining variables, 
$\theta \cup \pi \cup \theta_r$ is a stable model of $P$.
\item For a given assignment $\theta_r$ over remaining variables, if $\theta \cup \pi \cup \theta_r$ is a stable model of $P$, 
then $\pi$ is a stable model of $Q$.
\end{enumerate}
}

\newcommand{\cormain}{
Let the set of rules and constraints of $Q$ decompose into $k$ \aspsat programs $Q_1 =  (W_1, S_1, D_1), \ldots, Q_k = (W_k, S_k, D_k)$
where $W_i = \vars(S_i) \cup \vars(D_i)$ s.t. for any distinct $i,j$ in $1 \ldots k$, $W_i \cap W_j = \emptyset$.
Let the remaining variables be: $\VV_r = \VV \setminus (W_1 \cup \ldots \cup W_k \cup \vars(\theta))$ and
let $\pi_1, \ldots, \pi_k$ be complete assignments over $W_1, \ldots, W_k$ respectively.
\begin{enumerate}
\item If $\pi_1, \ldots, \pi_k$ are stable models of $Q_1, \ldots, Q_k$ resp., then for any assignment $\theta_r$ over the remaining variables, 
$\theta \cup \pi_1 \cup \ldots \cup \pi_k \cup \theta_r$ is a stable model of $P$.
\item For a given assignment $\theta_r$ over remaining variables, if $\theta \cup \pi_1 \cup \ldots \cup \pi_k \cup \theta_r$ is a stable model of $P$, 
then $\pi_i$ is a stable model of $Q_i$ for each $i \in 1 \ldots k$.
\end{enumerate}
}

\begin{theorem}
\label{th:main}
\thmain
\end{theorem}

\begin{corollary}
\label{cor:main}
\cormain
\end{corollary}

The first part of Theorem \ref{th:main} shows that we can solve
the justified residual program independently (as well as cache the result) and extend any of its stable model to a full stable model by assigning
any value to the remaining variables of the original program. 
The second part of the theorem establishes that any full stable model of the original program is counted since it is
an extension of the stable model of the residual program.
The corollary tells us that if the justified residual program decomposes into disjoint programs, then we can
solve each one of them independently, and multiply their counts to get the count for justified residual program.

\begin{example}
In Example \ref{ex:residual},
the justified residual program has only two stable models:
$\pi_1 = \{s, a, b, \neg t\}$ and $\pi_2 = \{t, a, b, \neg s\}$. It can be verified
that the only stable assignments extending $\theta$ of $P$ are $\theta \cup \pi_1 \cup \theta_{xyz}$ and
$\theta \cup \pi_2 \cup \theta_{xyz}$ where $\theta_{xyz}$ is any assignment on the standard 
variables $x,y,z$. Therefore, the total number of stable models below $\theta$ is $2 \times 2^{|\{x,y,z\}|} = 16$.

Now say we have another assignment $\theta' = \{a,b,c,d,u,e,\neg f, x\}$. It can be seen that it produces
the same justified residual program as that produced by $\theta$ for which we know the stable model count
is $2$. Furthermore, the set of remaining variables is $\{y,z\}$. Therefore, the number of stable
assignments below $\theta'$ is $2 \times 2^{|\{y,z\}|} = 8$.
\end{example}

In order to convert a model counter to a stable model counter, we can either modify its calculation of the residual program
as suggested by Theorem \ref{th:main}, or, we can modify the actual program and use its \emph{existing} calculation
in a way that residual of the modified program correctly models the \emph{justified} residual program.
%
Let us describe one such approach and prove that it is correct.  We post a copy of each founded
variable and each rule such that the copy variable only becomes true
when the corresponding founded variable is justified.  More formally,
for each founded variable $v$, we create a standard variable $v'$, add the
constraint $\neg v' \vee v$, and for each rule $a \leftarrow f_1 \wedge
\ldots f_n \wedge sn_1 \wedge \ldots \wedge sn_m$ where each $f_i$ is a
positive founded literal and each $sn_i$ is a standard or negative literal,
we add the clause $a' \vee \neg f_1' \vee \ldots \neg f_n' \vee \neg sn_1
\vee \ldots \vee \neg sn_m$.  Most importantly, we do not allow search to
take decisions on any of these introduced copy variables. Let this transformation
of a program $P$ be denoted by $\shadow(P)$.

We now show that it is correct to use the above approach for stable model counting.
\newcommand{\projectResultsSetup}{
For the following discussion and results, let $P=(\VV,R,C)$ be an \aspsat program,
$\shadow(P) = (W, R, D)$ be denoted by $Q$. Let $\pi$, $\pi_1$, $\pi_2$ be assignments over $W$ and $\theta$, $\theta_1$, $\theta_2$ be their projections
over non-copy variables ($\VV$). 
Let $\residual{Q}{\pi}$ (similarly for $\pi_1$, $\pi_2$) be a shorthand for $(\vars(\residual{R}{\theta}) \cup \vars(\residual{D}{\theta}),\residual{R}{\theta},\residual{D}{\theta})$. 
The results assume that assignments $\pi$, $\pi_1$, $\pi_2$ are closed
under unit propagation and unfounded set propagation, i.e., both propagators have been run until
fixpoint in the solver.

To prove the results, we define a function $\project$ that takes
the copy program $Q$ and $\pi$ and maps it to the justified residual program w.r.t. to the projection
of that $\pi$ on non-copy variables and then argue that $\residual{Q}{\pi}$ correctly models the justified residual program.
Formally, $\project(Q,\pi)$ is an \aspsat program $P' = (\VV', R', C')$ constructed as follows.
Add every constraint in $\residual{D}{\pi}$ that does not have a copy variable in $C'$.
For every constraint $v' \vee \neg f_1' \vee \ldots \neg f_n' \vee \neg sn_1
\vee \ldots \vee \neg sn_m$ in $\residual{D}{\pi}$, add the rule $v \leftarrow f_1 \wedge \ldots \wedge f_n \wedge sn_1 \wedge \ldots \wedge sn_m$ in $R'$.  
Let $U$ be the set of founded variables $v$ such that $v$ is true but $v'$ is unfixed in $\pi$. 
For every $v$ in $U$, add the constraint $v$ in $C'$. Define $\VV'$ as variables of $R'$ and $C'$.
}
\projectResultsSetup
Proposition \ref{prop:copyunsat}
proves that we cannot miss any stable model of the original program if we use the copy approach.

\newcommand{\propcopyunsat}{
If $\pi$ cannot be extended to any stable model of $Q$, then $\theta$ cannot be extended to any stable model of $P$.
}
\newcommand{\thprojection}{
$\justres{P}{\theta} = \project(Q,\pi)$.
}
\newcommand{\corcopycube}{
If $\residual{Q}{\pi}$ has no rules or constraints and there are $k$ unfixed variables, then $\theta$ is a stable model cube of $\justres{P}{\theta}$ of size $2^k$.
}
\newcommand{\corcopycache}{
If $\residual{Q}{\pi_1} = \residual{Q}{\pi_2}$, then $\justres{P}{\theta_1} = \justres{P}{\theta_2}$.
}
\newcommand{\corcopydecomp}{
If $\residual{Q}{\pi}$ decomposes into $k$ disjoint components $Q_1, \ldots, Q_k$, then
$\justres{P}{\theta}$ decomposes into $k$ disjoint components $P_1, \ldots, P_k$ such that
$P_i = \project(Q_i, \pi_i)$ where $\pi_i$ is projection of $\pi$ on $\vars(Q_i)$.
}

\begin{proposition}
\label{prop:copyunsat}
\propcopyunsat
\end{proposition}

Theorem \ref{th:projection} establishes that we can safely use $\residual{Q}{\pi}$ to emulate
the justified residual program $\justres{P}{\theta}$. 
Corollary \ref{cor:copycube} says that if we detect
a stable model cube of $\residual{Q}{\pi}$, then we also detect a stable model cube of the same size for the justified residual program. 
This corollary and Proposition \ref{prop:copyunsat} prove that the stable model count of the actual program is preserved.

\begin{theorem}
\label{th:projection}
\thprojection
\end{theorem}

\begin{corollary}
\label{cor:copycube}
\corcopycube
\end{corollary}

The next two corollaries prove that the copy approach can be used for caching dynamic decomposition
respectively.

\begin{corollary}
\label{cor:copycache}
\corcopycache
\end{corollary}

\begin{corollary}
\label{cor:copydecomp}
\corcopydecomp
\end{corollary}

\section{\Problog via Stable Model Counting}

In this section, we describe how we apply stable model counting in the probabilistic
logic programming solver \Problog~\cite{problog_journal}.
A probabilistic logic program is a collection of mutually independent \emph{random}
variables each of which is annotated with a probability, \emph{derived} variables, \emph{evidence} constraints 
and rules for the derived variables. The \emph{distribution semantics} \cite{distribution_semantics} says that for a given
assignment over the random variables, the values of the derived variables is given by the well-founded model.
Furthermore, the \emph{weight} of that world is equal to the product of probabilities of values
of the random variables. In our setting, it is useful to think of random variables, derived variables,
evidence constraints, and rules as standard variables, founded variables, constraints and rules respectively.
\Problog handles various inference tasks, but the focus of this paper is computing the marginal probability
of query atoms given evidence constraints. The probability of a query atom is equal to the sum of
weights of worlds where a query atom and evidence are satisfied divided by the sum of weights of worlds 
where the evidence is satisfied.

\begin{figure}[t]
  \centering
  \includegraphics[scale=0.1]{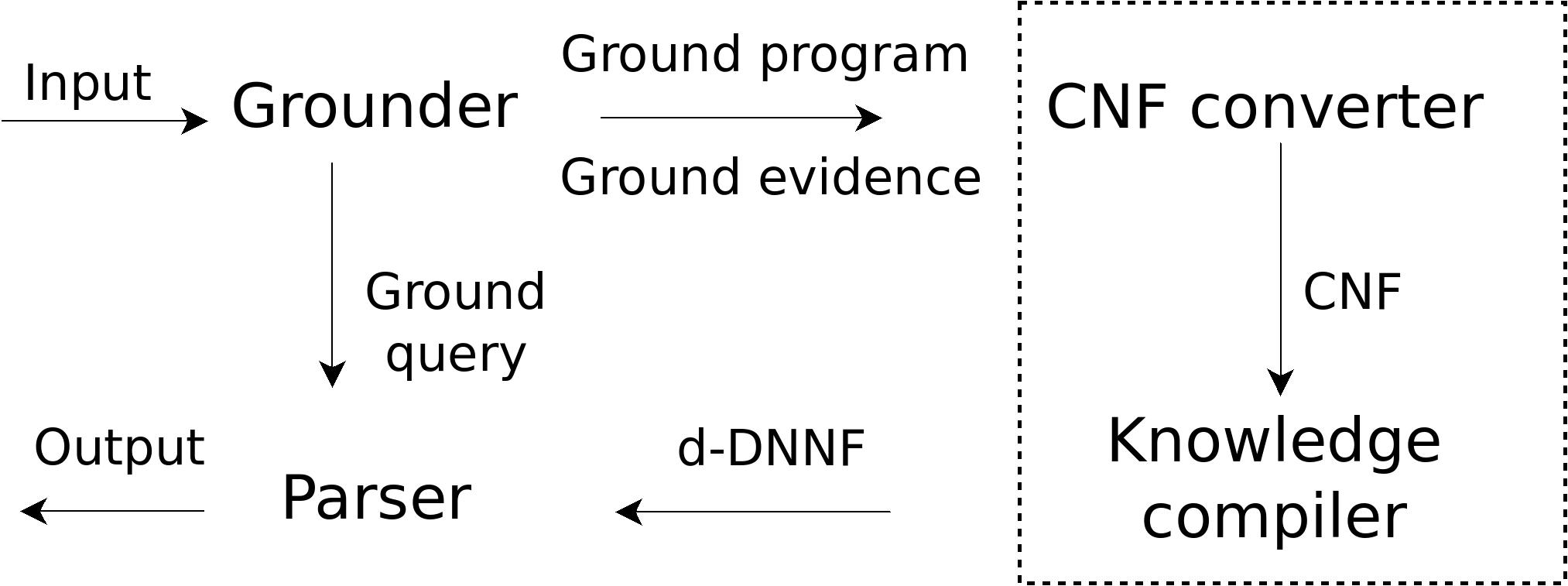}
  \caption {Execution of \Problog}
	\label{figure:problog}
\end{figure}

Figure \ref{figure:problog} shows the execution of a \Problog program. The input is a non-ground probabilistic
logic program which is given to the grounder that cleverly instantiates only parts of the program that are relevant
to the query atoms, similar to how \emph{magic set transformation} \cite{original_magic} achieves the same goal in logic programming.
The ground program and the evidence is then converted to CNF using the proof based encoding that we discussed earlier.
This CNF is passed on to a \emph{knowledge compiler} like \dsharp \cite{dsharp}. \dsharp is an extension of \sharpsat \cite{sharpsat}
where the DPLL-style search is recorded as \ddnnf.
The \ddnnf produced by the knowledge compiler is given to the parser of \Problog
along with the ground queries and probabilities of the random variables. The parser evaluates the probability of each
query by crawling the \ddnnf as described in \cite{problog_journal}.

Our contribution is in the components in the dotted box in Figure \ref{figure:problog}. We have implemented stable model counting
by extending the propositional model counter \sharpsat as described in the previous section.
Since \sharpsat is part of the knowledge compiler \dsharp, our extension of \sharpsat automatically extends \dsharp to a
stable model knowledge compiler.
The CNF conversion component in \Problog chain is replaced by a simple processing of the ground program and evidence
to our desired input format. In the first approach where the search is restricted to standard variables, the evidence needs
to be passed on to our stable model counter which posts a nogood (the current assignment of standard
variables) each time an evidence atom is violated. In approach given in Section~\ref{sec:modify}, however, we post each evidence as a unit
clause, much like \Problog does in its CNF conversion step. 
Including evidence in constraints in the second approach is safe since our residual program relies on the justified assignment only,
and propagation on founded literals that makes them true due to constraints does not change that.
Outside the dotted box in the figure, the rest of the \Problog logic remains the same.

\section{Experiments}

We compare the two approaches based on implementation of unfounded set detection as explained in Section \ref{sec:sm_counting} 
against the proof based encoding of \Problog. 
We use two well-studied benchmarks: \SmokersFriends \cite{problog}
problem and the graph reliability problem (\GraphReliability) \cite{graph_reliability} with evidence constraints.

In both problems, the graph is probabilistic. 
In \GraphReliability, the nodes are associated with probabilities while in
\SmokersFriends, the edges have probabilities. Naturally, for $n$ nodes, the number of random variables is in $O(n)$ and $O(n^2)$
for \GraphReliability~ and \SmokersFriends~ respectively.
Due to this, \GraphReliability~ has significantly more loops per random
variables in the dependency graph which makes it more susceptible to the
size problems of eager encoding.
We refer to the fixed search approach of Section~\ref{sec:fixed}
 as \ASProblogS and the proper 
integration of unfounded set detection through the
use of copy variables of Section~\ref{sec:modify} as \ASProblog.
All experiments were run on a machine running Ubuntu 12.04.1 LTS
with 8 GB of physical memory and Intel(R) Core(TM) i7-2600 3.4 GHz processor.

\newcommand{\best}[1]{\textbf{#1}}

\begin{table*}[t] \scriptsize
\setlength{\tabcolsep}{2pt}
\centering
\begin{tabular}{rc|rrrrrr|rrrrrrr|rrrrrrr}
\hline
\multicolumn{2}{c|}{Instance} & \multicolumn{6}{c|}{\Problog} & \multicolumn{7}{c|}{\ASProblog} & \multicolumn{7}{c}{\ASProblogS} \\
$N$ & $P$  & Time   & $V$    &  $C$     & $D$  & $A$  &  $S$  & Time   & $V$    &  $C$     & $D$  & $A$  &  $L$     & $S$ & Time   & $V$    &  $C$     & $D$  & $A$  &  $L$     & $S$ \\ [0.2ex]
\hline
10 & 0.5 & 11.33 & 2214 & 7065 & \best{199} & \best{7.68} & 1.21 & \best{1.08} & 72 & 226 & 233 & 8.88 & \best{13} & .\best{057} & 1.13 & \best{60} & \best{171} & 333 & 8.75 & 124 & .10 \\ [0.5ex]
11 & 0.5 & 115.75 & 6601 & 21899 & \best{353} & \best{8.61} & 7.62 & \best{1.11} & 86 & 283 & 382 & 9.76 & \best{23} & \best{.10} & 1.12 & \best{73} & \best{216} & 354 & 9.38 & 107 & .10 \\ [0.5ex]
12 & 0.5 & --- & 16210 & 55244 & --- & --- & --- & \best{1.20} & 101 & 348 & \best{675} & 10.81 & \best{21} & \best{.19} & 1.32 & \best{87} & \best{267} & 904 & \best{10.47} & 405 & .28 \\ [0.5ex]
13 & 0.5 & --- & 59266 & 204293 & --- & --- & --- & \best{1.41} & 117 & 414
& \best{1395} & 12.16 & \best{44} & \best{.41} & 2.61 & \best{102} &
\best{320} & 2737 & \best{11.33} & 1272 & 1.28 \\ [0.5ex]
\hline
15 & 0.5 & --- & --- & --- & --- & --- & --- & \best{2.05} & 142 & 514 & \best{3705} & 13.42 & \best{59} & \best{1.23} & 4.78 & \best{125} & \best{398} & 7542 & \best{12.88} & 2028 & 2.71 \\ [0.5ex]
20 & 0.5 & --- & --- & --- & --- & --- & --- & \best{31.82} & 246 & 966 & \best{83091} & 18.37 & \best{189} & \best{38.11} & 82.21 & \best{224} & \best{757} & 143188 & \best{18.31} & 32945 & 62.02 \\ [0.5ex]
25 & 0.25 & --- & --- & --- & --- & --- & --- & \best{22.44} & 225 & 800 &
\best{62871} & \best{18.70} & \best{231} & \best{27.23} & 53.63 & \best{198}
& \best{620} & 128534 & 19.55 & 41811 & 43.06 \\ [0.5ex]
\hline
30 & 0.1 & --- & --- & --- & --- & --- & --- & \best{3.71} & 168 & 468 & \best{7347} & \best{15.89} & \best{129} & \best{2.99} & 13.22 & \best{137} & \best{351} & 43968 & 19.31 & 2833 & 10.40 \\ [0.5ex]
31 & 0.1 & --- & 37992 & 115934 & --- & --- & --- & \best{2.84} & 171 & 473 & \best{5054} & \best{15.06} & \best{52} & \best{2.23} & 12.67 & \best{140} & \best{356} & 19585 & 17.53 & 1293 & 11.18 \\ [0.5ex]
32 & 0.1 & --- & --- & --- & --- & --- & --- & \best{7.93} & 185 & 528 & \best{17006} & \best{17.06} & \best{173} & \best{7.75} & 35.97 & \best{153} & \best{398} & 108916 & 21.42 & 5405 & 32.10 \\ [0.5ex]
33 & 0.1 & --- & --- & --- & --- & --- & --- & \best{25.13} & 191 & 533 & \best{67929} & \best{18.49} & \best{343} & \best{31.06} & --- & \best{157} & \best{403} & --- & --- & --- & --- \\ [0.5ex]
34 & 0.1 & --- & --- & --- & --- & --- & --- & \best{12.97} & 201 & 566 & \best{33338} & \best{19.41} & \best{155} & \best{14.66} & 112.27 & \best{165} & \best{429} & 324304 & 23.20 & 5502 & 124.21 \\ [0.5ex]
35 & 0.1 & --- & --- & --- & --- & --- & --- & \best{101.40} & 222 & 663 & \best{249512} & \best{21.78} & \best{1567} & \best{123.62} & --- & \best{186} & \best{503} & --- & --- & --- & --- \\ [0.5ex]
36 & 0.1 & --- & --- & --- & --- & --- & --- & \best{100.20} & 228 & 683 & \best{279273} & \best{21.41} & \best{1542} & \best{124.73} & --- & \best{190} & \best{518} & --- & --- & --- & --- \\ [0.5ex]
37 & 0.1 & --- & --- & --- & --- & --- & --- & \best{65.86} & 227 & 659 & \best{159056} & \best{20.55} & \best{658} & \best{77.57} & --- & \best{188} & \best{499} & --- & --- & --- & --- \\ [0.5ex]
38 & 0.1 & --- & --- & --- & --- & --- & --- & --- & 240 & 712 & --- & --- & --- & --- & --- & \best{200} & \best{540} & --- & --- & --- & --- \\ [0.5ex]
\hline \\[0.2ex] 
\end{tabular}
\caption{Comparison of \Problog, \ASProblog, and \ASProblogS on the Graph Reliability problem}
\label{table:graph_reliability}
\end{table*}

Table \ref{table:graph_reliability} shows the comparison between \Problog, \ASProblog and \ASProblogS
on \GraphReliability on random directed graphs. The instance is specified by $N$, the number of nodes, and $P$,
the probability of an edge between any two nodes.
The solvers are compared on the following parameters: time in seconds (Time), number of variables and
clauses in the input program of \dsharp ($V$ and $C$ resp.), number of decisions ($D$), average decision
level of backtrack due to conflict or satisfaction ($A$), the size in megabytes of the \ddnnf produced by \dsharp
($S$), and for \ASProblog and \ASProblogS, the number of loops produced during the search ($L$).
Each number in the table represents the median value of that parameter from 10 random instances of the size in the row.
The median is only defined if there are at least (6) output values.
A `---' represents memory exhaustion or a timeout of 5 minutes, whichever occurs first.
A `---' in columns Time, $D$, $A$, $L$, $S$ means that the solver ran out of memory
but the grounding and encoding was done successfully, while a `---' in all columns of a solver means that
it never finished encoding the problem. We show the comparison on three types of instances: small graphs
with high density, medium graphs with high to medium density, and large graphs with low density.

Clearly \ASProblog and \ASProblogS are far more scalable than \Problog.
While \Problog requires less search (since it starts with all loop formulae
encoded) the overhead of the eager encoding is prohibitive.
For all solved instances, \ASProblog has the best running time and 
\ddnnf size, illustrating that the search restriction of \ASProblogS 
degrades performance significantly. 
While the encoding for \ASProblogS is always smallest, the encoding with
copy variables and rules of \ASProblog is not significantly greater,
and yields smaller search trees and fewer loop formulae.
It is clearly the superior approach.

\begin{figure}[t]
  \centering
  \includegraphics[scale=0.45]{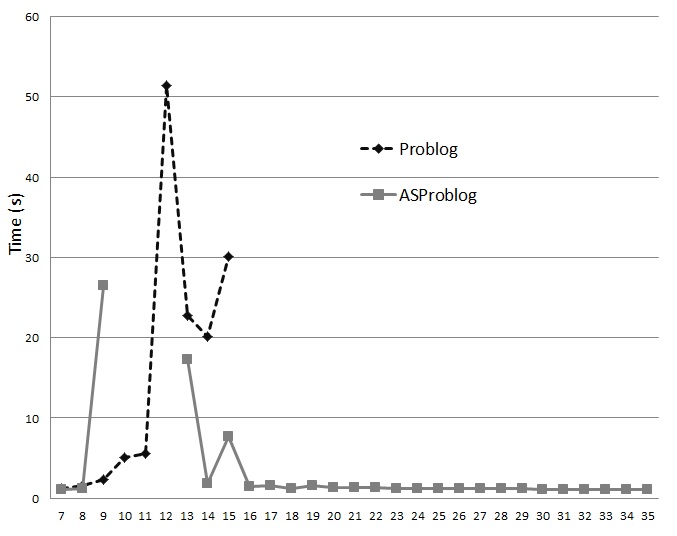}
  \caption {\SmokersFriends{} with $31$ random variables}
	\label{figure:smokers_friends}
\end{figure}

Figure \ref{figure:smokers_friends} compares the performance of \Problog and \ASProblog on \SmokersFriends
when the number of random variables is fixed to 31 and the problem size is increased. 
In the problem description, there are two sets of random variables, the $\stress$
and the $\influences$ variables. The first one exists for each person in the graph, while the
latter exists for every edge in the graph. In our setting, for an instance with $n$ persons,
the number of $\influences$ random variables is equal to $31 - n$. The rest of the $\influences$
variables are fixed to true or false at run time.
For the smallest instances of sizes 7 and 8, \Problog and \ASProblog have similar performance.
For instances 9 to 12, \Problog does better than \ASProblog where the latter cannot solve instances
11 and 12 due to memory exhaustion.
The reason is that the complete encoding in \Problog propagates better and the extra
unfounded set check at each node in the search tree in \ASProblog does not pay off. 
But as the number of people increases and the number of probabilistic edges becomes less, the problem becomes easier for \ASProblog
but not for \Problog.
The reason is that by fixing the probabilistic edges, we are just left with
$n$ external rules, and many internal rules, making many founded variables logically equivalent to each other.
In the last instance, the number of loop formulas required for the problem is only one! 
Our lazy approach benefits from this structure in the problem, while \Problog does not.
Our experiments with the same range of instances but with number of random variables
fixed to 33 and 35 show similar behaviour of \Problog and \ASProblog where initially,
\Problog does better, followed by hard instances for both, and finally, \ASProblog detecting
the structure and solving the last few instances in less than 2 seconds.

\ignore{
instance,problog_31,asproblog_31,problog_33,asproblog_33,problog_35,asproblog_35
7,1.21,1.12,1.13,1.07,1.18,1.12,
8,1.61,1.25,1.31,1.26,1.30,1.17,
9,2.28,26.44,2.04,9.74,1.44,39.85,
10,5.05,---,8.82,---,9.23,---,
11,5.54,---,3.27,---,---,---,
12,51.43,---,7.88,56.60,---,---,
13,22.77,17.22,---,---,---,---,
14,20.16,1.83,---,---,---,---,
15,30.07,7.64,218.04,38.20,163.54,---,
16,---,1.50,---,34.40,---,---,
17,---,1.58,---,---,---,171.49,
18,---,1.24,---,4.41,---,---,
19,---,1.62,---,1.36,---,---,
20,---,1.29,---,18.42,---,---,
21,---,1.30,---,1.36,---,126.30,
22,---,1.28,---,1.36,---,1.41,
23,---,1.27,---,1.33,---,1.45,
24,---,1.24,---,1.31,---,1.35,
25,---,1.24,---,1.31,---,1.40,
26,---,1.22,---,1.31,---,1.37,
27,---,1.20,---,1.26,---,1.34,
28,---,1.18,---,1.25,---,1.34,
29,---,1.16,---,1.23,---,1.34,
30,---,1.14,---,1.21,---,1.31,
31,---,1.11,---,1.19,---,1.29,
32,---,1.08,---,1.16,---,1.27,
33,---,1.09,---,1.12,---,1.24,
34,---,1.09,---,1.08,---,1.20,
35,---,1.09,---,1.09,---,1.15,
}

\section{Conclusion}

Stable model counting is required for reasoning about probabilistic logic programs with positive recursion in their rules.
We demonstrate that the current approach of translating logic programs eagerly to propositional theories is not scalable because
the translation explodes when there is a large number of recursive rules in the ground program. 
We give two methods to avoid this problem which enables reasoning about significantly bigger probabilistic logic programs.

\bibliographystyle{aaai}
\bibliography{paper}

\noproofs{
\clearpage

\newtheorem{apxTheorem}{Theorem}
\newtheorem{apxCorollary}[apxTheorem]{Corollary}
\newtheorem{apxProposition}[apxTheorem]{Proposition}

\appendix
\section{Proofs of theorems and their corollaries}

\begin{apxTheorem}
\thmain
\end{apxTheorem}

\begin{proof}
Let $J = \JA(P, \theta)$.
Note that there cannot be a founded variable in the remaining variables $\VV_r$ since if a founded variable is not true in $\theta$
and does not have a rule in $S$, then it must be false in $\theta$ due to the given assumption.

The key point is to view $R$ as two separate sets of rules, $S$ and $R \setminus S$ and argue that we can treat them separately
for the purpose of least models. If there is any assignment that extends $\theta$, then from $R \setminus S$, we can
safely delete the rules whose bodies intersect with $W$ or $\theta_r$ as these rules are redundant
since all assignments in $J_0(\theta)$ are sufficient to imply the founded literals in $J$. 
Furthermore, the least assignment of reduct of $R \setminus S$ w.r.t. any assignment that extends $\theta$ will be exactly equal
to the founded literals in $J$. Therefore:

$$ \Least(R^{\theta \cup \pi \cup \theta_R}) =_{\Founded} J \cup \Least(S^{\pi \cup \theta_R})$$

Since $\theta_R$ does not have any founded variables and we just argued that we can delete the rules in $R \setminus S$
that have any variable from $\theta_R$ and $\theta_R$ is completely disjoint from $S$ by definition, 
we can simplify the above equality to:
$ \Least(R^{\theta \cup \pi}) =_{\Founded} J \cup \Least(S^{\pi})$.

\begin{enumerate}
\item We are given that $\pi$ is a stable model of $Q$, i.e., $\pi \models S$,
$\pi \models D$ and $\Least(S^{\pi}) =_{\Founded} \pi$. It is easy to show that
this implies that $\theta \cup \pi \models R \cup C$. Moreover, from the above
equality, we get $\Least(R^{\theta \cup \pi}) =_{\Founded} J \cup \pi$.
Since, due to constraints added in $D$, $\theta$ and $\pi$ are consistent on founded variables in $\VV \setminus \vars(J)$,
we get $\Least(R^{\theta \cup \pi}) =_{\Founded} \theta \cup \pi$.
It is easy to see that any assignment $\theta_R$ can be used to extend $\theta \cup \pi$
without affecting satisfiability or the least model, which means that $\theta \cup \pi \cup \theta_R$ is a stable model of $P$.

\item We are given that $\theta \cup \pi \cup \theta_R$ is a stable model of $P$.
By definition of residual programs, we know that $\theta \models C \setminus D$ and
the intersection of variables in $D$ and $\theta$ is empty, which means that if $D$ is non-empty,
then $\theta$ is not sufficient to satisfy $D$ which implies that $\pi \models D$ (if $D$ is empty,
then trivially $\pi \models D$). A similar
argument for the case $\pi \models \residual{R}{\theta}$ can be made.
Given that $\theta$ and $\pi$ are consistent, we can also see that $\pi \models S \setminus \residual{R}{\theta}$
which means that $\pi \models S$. For least model, from the equality that we discussed previously,
we are given that $\theta \cup \pi =_{\Founded} J \cup \Least(S^{\pi})$. Again,
since $\theta \cup \pi =_{\Founded} J \cup \pi$, we can derive that $\pi =_{\Founded} \Least(S^{\pi})$
which means that $\pi$ is a stable model of $Q$.
\end{enumerate}
\end{proof}

\begin{apxCorollary}
\cormain
\end{apxCorollary}
\begin{proof}
Theorem \ref{th:main} says that the justified residual program can be solved in isolation and its results
can be combined with the parent program. Furthermore, since all ASP programs $Q_1, \ldots, Q_k$ are completely
disjoint, both 1 and 2 follow from the \emph{Module Theorem} (Theorem 1) as given in \cite{modularity_asp} which says
that two mutually compatible assignments that are stable models of two respective programs can be joined
to form a stable model of their union program and conversely, a stable model of the combined program can be split
into stable models of individual programs, as long as there are no positive interdependencies between
the two programs. Ours is a simple special case of Corollary 1 in \cite{modularity_asp} where all programs and their sets of variables are completely disjoint.
\end{proof}

\projectResultsSetup

\begin{apxProposition}
\propcopyunsat
\end{apxProposition}
\begin{proof}[Proof sketch]
Say $\theta$ has an extension $E$ that is a stable model of $P$. We can show that
running unit propagation on $Q$ and $E$ yields a solution of $Q$ that is an extension of $\pi$, which contradicts what is given.
\end{proof}

\begin{apxTheorem}
\thprojection
\end{apxTheorem}
\begin{proof}
Recall the definition of justified assignment:
$\JA(P, \theta) = J_0(\theta) \cup \{v \in \Founded | v \in \theta, v \in \Least(\residual{R}{J_0(\theta)})  \}$.
Also recall that any copy constraint $r'$ in $Q$ has the form:
$v' \vee \neg f_1' \vee \ldots \neg f_n' \vee \neg sn_1 \vee \ldots \neg sn_m$ and by definition, each $r'$ is the copy
of a rule $r$ in $P$, which is $v \leftarrow f_1 \wedge \ldots f_n \wedge sn_1 \wedge \ldots \wedge sn_m$.
Recall that $v', f_1', \ldots, f_n'$ are copy variables of $v, f_1, \ldots, f_n$ respectively,
$f_1, \ldots, f_n$ are positive literals in $r$ and $sn_1, \ldots, sn_m$ are either standard or negative literals in $r$.
We show that the sets of rules, constraints, and variables of $\project(Q,\pi)$ and $\justres{P}{\theta}$
are equal, therefore, they are equal. We begin by reasoning about the sets of rules.

Let us focus on the \emph{seed} of the justified assignment $J_0(\theta)$. It is easy to see that
$J_0(\pi) \cap \VV = J_0(\theta)$. Since $r'$ and $r$ share all standard and negative literals, their
residuals w.r.t. these literals will be simplified in exactly the same way, i.e., if 
$\residual{r'}{J_0(\pi)} = v' \vee \neg f_1' \vee \ldots \neg f_n' \vee \neg sn_1 \vee \ldots \neg sn_k$,
then $\residual{r}{J_0(\theta)} = v \leftarrow f_1 \wedge \ldots f_n \wedge sn_1 \wedge \ldots \wedge sn_k$.
The core point of the proof is that running unit propagation on the set of copy rules $\residual{r'}{J_0(\pi)}$ is
analogous to computing $\Least(\residual{R}{J_0(\theta)})$. Each application of unit propagation on $r'$ that derives
$v'$ must also derive $v$ in $\residual{r}{J_0(\theta)}$. This means that if, due to this propagation,
the set of copy variables $J'=\{v_1', \ldots, v_j'\}$ is derived, then $\JA(P, \theta) \setminus J_0(\theta) = \{v_1, \ldots, v_j\}$.
Furthermore, since no decisions on copy variables are allowed, and there is no other constraint that can possibly
derive a literal $v'$, unit propagation cannot derive any other positive copy literal.
Now, let us view $\project(Q,\pi)$ as individual applications of a function $\projectr$ on each copy rule $r'$ to produce $r$.
We can see that $\projectr(\residual{r}{J_0(\pi) \cup J'}) = \residual{r}{\JA(P, \theta)}$.
Therefore, the set of rules in $\project(Q,\pi)$ and $\justres{P}{\theta}$ is exactly the same.

From above, it also follows that the set $U$ in the construction of $\project(Q,\pi)$ and the set $U$ in Definition \ref{def:justres}
are equal. Since $\theta$ is just the restriction of $\pi$ on non-copy variables, the residual of any constraint that has non-copy variables only is the same
and since the consistency constraints due to $U$ are also the same, the set of constraints $C'$ and $C$ are equal. Since the rules
and constraints are equal in $\project(Q,\pi)$ and $\justres{P}{\theta}$, their variables are also equal.
\end{proof}

\begin{apxCorollary}
\corcopycube
\end{apxCorollary}
\begin{proof}
Since $\residual{Q}{\pi}$ is empty, $\project(Q,\pi)$ is also empty,
and since $\justres{P}{\theta} = \project(Q,\pi)$, this means that $\justres{P}{\theta}$ is a stable model cube.
Furthermore, we can show that all founded variables and their copies must be fixed in $\residual{Q}{\pi}$,
which means all $k$ variables must be standard variables, therefore, the size of stable model cube of $\justres{P}{\theta}$
is $2^k$.
\end{proof}

\begin{apxCorollary}
\corcopycache
\end{apxCorollary}
\begin{proof}
Follows directly from Theorem \ref{th:projection} and definition of $\project$ function.
\end{proof}

\begin{apxCorollary}
\corcopydecomp
\end{apxCorollary}
\begin{proof}
Note that the disjointness of components of $\residual{Q}{\pi}$ is really determined by its constraints, and not
its rules. The residual rules are always stronger (have fewer variables) than their respective residual copy constraints,
because it is possible that a founded variables $v$ is true in $\pi$ but its copy variable $v'$ is unfixed. The opposite
is not possible and moreover, standard and negative literals are shared in a rule and its corresponding copy constraint. 
This property of residual rules is important since $\project(Q,\pi)$ works by projecting each individual copy rule to its original form in $\justres{P}{\theta}$ and completely ignores the residual set of rules in $\residual{Q}{\pi}$.

It is clear from definition of $\project(Q,\pi)$ that the projection of each rule merely replaces the copy variables
with their corresponding founded variables. This means that we can take each disjoint component of $\residual{Q}{\pi}$
and map it to its counterpart in $\justres{P}{\theta}$ using the projection function. 
Non-copy constraint translation in $\project(Q,\pi)$ is also straight-forward
and cannot combine multiple disjoint components in $\residual{Q}{\pi}$ to one component in $\justres{P}{\theta}$ or split one component
in $\residual{Q}{\pi}$ to multiple components in $\justres{P}{\theta}$. 
Finally, addition of the unary clauses for all variables in $U$ (as used in
the definition of $\project(Q,\pi)$) also does not affect the components; by definition, a variable $v$ in $U$
must have at least one copy constraint $v' \vee \neg f_1' \vee \ldots \neg f_n' \vee \neg sn_1 \vee \ldots \neg sn_m$
which will be projected to $v \leftarrow f_1 \wedge \ldots f_n \wedge sn_1 \wedge \ldots sn_m$ in $\justres{P}{\theta}$. Adding
$v$ as a constraint does not affect the component $P_i$ in which this projected rule appears in $\justres{P}{\theta}$.
\end{proof}
}

\end{document}